\documentclass[12pt]{colt2015}
\usepackage{hyperref}
\usepackage{url}
\usepackage{graphicx}
\usepackage{amsmath}
\usepackage{amsfonts}
\usepackage{amsopn}
\usepackage{ifthen}
\usepackage{natbib}
\usepackage{epstopdf}
\usepackage{bbm}
\usepackage{algorithm2e}
\usepackage[noend]{algpseudocode}
\usepackage[francais,polish,english]{babel}
\usepackage{polski}
\usepackage{color}

\newcounter{nbdrafts}
\makeatletter
\newcommand{\checknbdrafts}{
\ifnum \thenbdrafts > 0
\@latex@warning@no@line{**********************************************************************}
\@latex@warning@no@line{* The document contains \thenbdrafts \space draft note(s)}
\@latex@warning@no@line{**********************************************************************}
\fi}

\makeatother

\newcommand{\qed}{\nobreak \ifvmode \relax \else
      \ifdim\lastskip<1.5em \hskip-\lastskip
      \hskip1.5em plus0em minus0.5em \fi \nobreak
      \vrule height0.75em width0.5em depth0.25em\fi}

\DeclareMathOperator*{\argmin}{arg\,min}

\title{A provably convergent alternating minimization method for mean field inference}

 \coltauthor{\Name{Pierre Baque} \Email{pierre.baque@epfl.ch}\\
 \addr EPFL IC ISIM CVLAB, CH-1015 Lausanne, Switzerland
 \AND
 \Name{Jean-Hubert Hours} \Email{jean-hubert.hours@epfl.ch}\\
 \addr EPFL STI IGM LA3, CH-1015 Lausanne, Switzerland
  \AND
 \Name{Francois Fleuret} \Email{francois.fleuret@idiap.ch}\\
 \addr Idiap Research Institute, CH-1920 Martigny, Switzerland
  \AND
 \Name{Pascal Fua} \Email{pascal.fua@epfl.ch}\\
 \addr EPFL IC ISIM CVLAB, CH-1015 Lausanne, Switzerland
 }

\begin{document}

\maketitle

\begin{abstract} 
Mean-Field is an efficient way to approximate a posterior distribution in complex graphical models and constitutes the most popular class of Bayesian variational approximation methods. In most applications, the mean field distribution parameters are computed using an alternate coordinate minimization. However, the convergence properties of this algorithm remain unclear. In this paper, we show how, by adding an appropriate penalization term, we can guarantee convergence to a critical point, while keeping a closed form update at each step. A convergence rate estimate can also be derived based on recent results in non-convex optimization.
\end{abstract}

\section{Introduction} 

In many situations when a posterior  distribution $P$ over variables $X$ depends
on  a   complex  model,  exact   inference  is  not  possible   and  variational
inference~\citep{Jordan2008,Attias00avariational} is  a widespread  approach to
approximating it.   This technique is used  in domains such as  Computer Vision,
Natural   Language   Processing,   and   large   scale   Data   Processing,   as
in~\cite{POMOriginal,VariationalApproximationLanguage,VariationalApproximationMarket}.

Mean  field  variational   inference  methods  approximate  $P$   by  a  product
distribution  $Q$,  which  means  looking  for  the  distribution  $Q$  among  a
restricted class of  product distributions. The quality of  the approximation is
measured in terms of the Kullback-Leibler divergence between $P$ and $Q$.
This turns the mean field problem into a non-convex minimization problem.\\

The  most  popular  approach  to   solving  it  is  the  alternate  minimization
approach~\citep{MLBishop,PGMKoller},  also  known  as the  Variational  Message
Passing  algorithm~\citep{Winn2005}  in  the machine learning community.   The  Kullback-Leibler
divergence  is minimized  coordinate by  coordinate in  a pre-determined  order,
until  convergence.   The  main  advantage   of  this  algorithm  is  that  the
coordinate-wise   minimum   can   be   computed   in   closed   form   at   each
step. Furthermore,  the procedure can  be parallelized  in most cases,  as shown
in~\citet[p. 21]{Bertsekas}.\\

However, convergence is  not always guaranteed for the general alternate minimisation in the non-convex case. One can find  examples where the
procedure endlessly loops between several equivalent local minima which become
cluster points of  the minimization sequence, as shown  in~\cite{Powell}.  More
specifically, convergence  can be  proven in  some cases~\citep{Tseng01}  but not
all. More precisely,  the objective function always decreases but  that does not
preclude oscillations  in the variables  and there is  no formal proof  that the
alternating minimization algorithm for variational  inference will never loop as
in the Powell example.\\

Our contribution is the introduction of a special purpose proximal regularisation term at each step of
the  minimization  that provably  enforces  convergence. It dampens  potential oscillations while  preserving the
simplicity of the algorithm whose updates  are still computed in closed form. We use a recent result  from~\citep{Bolte2013} to prove
formally that it is indeed the case.\\

It is important to understand that, as the objective function is non-convex, our proximal algorithm doesn't always converge to the same minimum as the classical fixed point algorithm. However, the solution found has no reason to be better or worse. Furthermore, the proximal term can be chosen arbitrarily small through the parameter $\lambda$. Therefore, by choosing a small $\lambda$, one can be make the new proximal algorithm follow a trajectory which is arbitrarily close to the trajectory of the alternate minimisation.

\begin{table}[h]
  \begin{center}
    \caption{\label{tab:notations}Notations}

   \vspace*{1em}

    \framebox{
      \hspace*{0.5em}
      \begin{minipage}{0.9\columnwidth}
        \begin{itemize}
          \setlength{\itemsep}{0.4em}
        \item{$\|.\|$ is the Euclidean norm in $\mathbb{R}^N$.}

	\item{For a differentiable function f, $\nabla f$ its gradient.}

	\item $Q=\{q_1, \dots,q_N\}$  is either the probability distribution on $N$ independent Bernoulli variables $\{X_1,\dots,X_N\}$ or a vector in $[0,1]^N$
	\item If $f$ is a function and X a random variable  $E_{Q}\left(f(X)\right)$ is the expected value of $f(X)$ under probability Q.
	\item If $X = \{X_1,\dots,X_N\}$ are independent Bernoulli variables under Q, $E_{Q_{\setminus i}}(f(X)|X_i=a)$ is the expected value of $f(X)$, given $X_i = a$.
	\item If $\left\{X^t\right\}$ is a convergent sequence in $\mathbb{R}^N$, its unique limit point is denoted by~$\overline{X}$.

        \end{itemize}
      \end{minipage}
    }
  \end{center}
\end{table}

\section{Variational inference problem}

We first recall the general formulation of KL divergence minimization problems as they appear in variational inference problems.\\
We assume that we are working with N random variables $\{X_1, \dots,X_N\}$ whose posterior distribution is taken from the exponential family (as in~\cite{Winn2005} and~\cite{MLBishop}), with marginal priors $p^0_i(X_i)$. The energy function is denoted by $ \Psi$. We make the important assumption that it is bounded.\\

\begin{equation*}
P(X)= \dfrac{1}{Z}e^{-\Psi(X)} \prod \limits_{i} p^0_i(X_i)
\end{equation*}
Where $Z$ is a normalisation factor.\\
Following the traditional mean field approach~\citep{MLBishop, PGMKoller}, we are now trying to get a tractable representation $Q(X) $ of this probability distribution $ P(X)$. By tractable, we mean a distribution that we can easily manipulate, sample from and calculate expectancies. We are therefore approximating $P$ by $Q$, among the product distributions. $Q(X)   = \prod \limits_{i} Q_i(X_i) $. It means, that we will look for $Q$ which is closest to $P$ in the sense of the KL divergence $KL(Q\| P)$. \\

For the sake of simplicity, it is assumed in the following that $X_i$ are Bernoulli variables (i.e in $\{X_1, \dots,X_N\} \in \{0,1\}^N$). However, we could easily work with non-binary random variables.\\
 Therefore, the approximating distribution can be written as : $Q(X)   = \prod \limits_{i \in \{1, \dots,N\}} Q_i(X_i) =\prod \limits_{i \in \{1, \dots,N\}} q_i^{X_i}(1-q_i)^{1-X_i}$.\\

The general form of the KL divergence is :
 \begin{equation}
KL(Q\|P)  = \sum_{x \in \{{0,1}\}^N} Q(x) \log\left( \dfrac{Q(x)}{P(x)}\right)
  \end{equation}

  Which we can rewrite as the sum of a multivariate polynomial and univariate convex functions :

 \begin{equation}
  KL(Q\|P)  = \sum_{x \in \{{0,1}\}^N} \Psi(x) \prod \limits_{i} Q_i(x_i) +  \sum \limits_{i} f_i(q_i)
  \end{equation}
  Where :
 \begin{equation}
f_i (q_i) = \log\left(\dfrac{1-q_i}{1-p^0_i}\right)(1-q_i) + \log\left(\dfrac{q_i}{p^0_i}\right)(q_i)
  \end{equation}

We introduce the functions $G(\{q_1, \dots,q_N\})$ and $\Omega(\{q_1, \dots,q_N\})$ such that:
 \begin{align}
 \label{G}
G(\{q_1, \dots,q_N\})&= KL(\{q_1, \dots,q_N\} \| P)    \\
&=  \sum_{x \in \{{0,1}\}^N} \Psi(x) \prod \limits_{i} Q_i(x_i) +  \sum \limits_{i} f_i(q_i)\\
 &=\Omega(\{q_1, \dots,q_N\}) +  \sum \limits_{i} f_i(q_i)
  \end{align}
Where :
\begin{definition}
\label{OmegaDef}
\begin{equation}
\Omega(\{q_1, \dots,q_N\}) := E_{Q}(\Psi(X))
\end{equation}
\end{definition}

The KL divergence minimization is thus the following :
 \begin{align}
 \label{eq:OptimProb}
 \argmin \limits_{\{q_1,...,q_N\} } G(\{q_1,...,q_N\})  
   \end{align} 
  This problem is obviously non-convex as it involves a sum of multiple products. Therefore, finding a global minimum can be cumbersome in large dimensions. In the next section, an algorithm which yields a sequence converging to a first-order critical point is introduced. An estimate of the local convergence rate can also be derived, based on~\cite{Bolte2013}. 
  
  \section{Proximal alternate minimisation algorithm}
In this section, we derive a tractable algorithm that converges to a first order
stationary  point of  the  problem of  Eq.~\ref{eq:OptimProb}, with  convergence
guaranties and a provable asymptotic convergence rate.  \\
Although the alternate minimisation algorithm produces a decreasing sequence of objective functions, there is a-priori no guarantee that the variable sequence actually converges as demonstrated by  \cite{Absil05Convergenceof}.~\cite{Powell} shows examples of minimisation problems for which a coordinate descent method fails to converge.\\
  However, we show in this paper, that, by adding a proximal regularisation, we can use the Kurdyka-\Lpb{}ojasiewicz inequality and recent work by Attouch and Bolte to prove convergence. The  specific form of  the penalty term lets  us retain the ability  to compute the  updates in closed form  in the case of variational inference.\\
  \paragraph{Regularisation}
  We are using a regularisation function which is the KL divergence between the one dimensional iterates. During the iterations, this proximal function $l(q,q_0)$ penalises the variables which are too different from their previous value. 
   \begin{align}
   \label{f_definition}
 l(q,q_0) &=  q\log\left(\dfrac{q}{q_0}\right) +  (1-q) \log\left(\dfrac{1-q}{1-q_0}\right)\\
 &=KL(B(Q)||B(Q_0)\\
  \end{align}
  Given $q_0$, $l(q,q_0)$ is strongly convex with regards to $q$, positive, continuous on $]0,1[^2$. Its minimum is 0 for $q=q_0$. \\
It is worth noting that the derivative of this function on the minimisation variable $x$ is simple as well and can be written as follows : 
    \begin{equation}
    \label{lprime}
{l}^{\prime}(q,q_0) =  \dfrac{ \partial l(q,q_0)}{\partial q} =  \log\left(\dfrac{q}{q_0}\right) - \log\left(\dfrac{1-q}{1-q_0}\right)
  \end{equation}

  \begin{remark}
  If $X$ is not binary, then, we just replace $l$ by the KL divergence between discrete random variables.
  \end{remark}

    \paragraph{Proximal alternate minimization procedure}
  We are looping through the variables, minimizing the objective over one variable at a time, the others staying fixed (see alg.~\ref{algo}), with the following update rule:

   \begin{equation}
   \label{procedure}
q^{t+1}_i = \argmin \limits_{q} \{G(\{q^{t+1}_1, \dots,q^{t+1}_{i-1},q,q^{t}_{i+1},q^{t}_N\})  + \lambda l(q,q^{t}_i) \}
  \end{equation}
   \begin{algorithm}

\KwIn{A prior distribution $\{p^0_1,...,p^0_N\}$ and a KL function $G$}
\KwOut{A MF distribution $\{q_1,...,q_N\}$ }

\emph{Initialisation to the prior} : \\
$\{q_1,...,q_N\}  \gets \{p^0_1,...,p^0_N\}$ \\
\emph{Loop untill convergence :}\\
\While {$\| \nabla G(\{q_1,\dots ,q_N\}) \|> \epsilon$}{
	\For{$X_i$ in $\{1,...,N\}$}{
 		$q_i  \gets \argmin \limits_{q} G(\{q_1,...,q,..., ,q_N\}) + \lambda l(q,q_i)$
	}
}
\Return{$\{q_1,...,q_N\}$}\;
\caption{KL Proximal alternate minimisation}\label{Algo}
\label{algo}
\end{algorithm}

  The main advantage of our penalization method (e.g $l(q,q^{t}_i))  $ over the quadratic one (e.g  $\|q-q^{t}_i\|^2$ as in~\cite{AttouchDemo}) is that the update is computed in closed form. Indeed, the minimization is differentiable and convex on $q$. Therefore, the first order one dimensional minimality condition gives :

    \begin{equation}
    \label{minformula}
 q_i^{t+1} =   \dfrac{1}{1 + \exp\left(\dfrac{1}{1+\lambda}\left[E_{Q^t_{\setminus i}}(\Psi(X)| X_i = 1) - E_{Q^t_{\setminus i}}(\Psi(X)| X_i = 0) + \log\left(\dfrac{1-p^0_i}{p^0_i}\right)+ \lambda \log\left(\dfrac{1-q^t_i}{q^t_i}\right)\right]\right)}
 \end{equation}
 With the notation:
    \begin{equation*}
  Q^t_{\setminus i} = \{q^{t+1}_1, \dots,q^{t+1}_{i-1},q^{t}_{i+1},q^{t}_N\},
  \end{equation*}

    \begin{remark}
  If $X$ is not binary, a similar closed form minimum is easily obtained by introducing a Lagrange multiplier as in~\cite{Beal}.
  \end{remark}

   \section{Convergence of the algorithm}
   Our analysis is along the lines of~\cite{AttouchDemo}, using the Kurdyka-\Lpb{}ojasiewicz inequality as the key tool in our proof.
   \subsection{A general convergence result}
   \begin{definition}[Kurdyka-\Lpb{}ojasiewicz Property]
\label{KLDefinition}
A differentiable function f is said to have the Kurdyka-\Lpb{}ojasiewicz property at $\overline{x}$, if there exists $\eta \geq 0$, a neighborhood U of $\overline{x}$ and a continuous concave functions $\phi : [0,\eta) \rightarrow \mathbb{R}_{+}$, such that : \\
-- $ \phi(0) =0$ \\
-- $\phi$ is $C^1$ on $(0,\eta)$\\
-- $\forall s \in (0,\eta), \phi^{\prime}(s) \geq 0$.\\
-- $\forall \overline{x} \in U \cap [f(\overline{x}) \geq f \geq f(\overline{x}) + \eta]$, the following inequality, called Kurdyka-\Lpb{}ojasiewicz inequality holds:\\
\begin{equation}
\phi^{\prime}(f(x) - f(\overline{x}))\|\nabla f(x) \| \geq 1
\end{equation}

\end{definition}

   \begin{lemma}
\label{Central}
Let F be any differentiable function from $\mathbb{R} $ to $\mathbb{R}^N $, and $X^t$ a bounded sequence which has the three following properties :\\
(i) Sufficient decrease :\\
	$\exists \lambda$ such that, $\forall t \geq 0$
	\begin{equation}
		F(X^{t+1}) + \dfrac{\lambda}{2} \|X^{t+1} - X^{t}\|^2 \leq F(X^{t})
	\end{equation}
(ii) Gradient bound :\\
	$\exists C$ such that, $\forall t \geq 0$
	\begin{equation}
		\| \nabla F(X^t) \| \leq C  \|X^{t+1} -X^{t} \|
	\end{equation}
(iii) The function $F$ has the Kurdyka-\Lpb{}ojasiewicz property at all its critical points, with  $\phi(s) = cs^{1-\theta}$ and  $ \theta \in ]0, 1[ $.\\
Then, the sequence $X^t$ converges to a stationary point of $F$ that we denote $\overline{X}$. Moreover, the following convergence rates apply (depending on $\theta$).

    (a) If $ \theta \in \left]0, \dfrac{1}{2}\right] $, then $ \exists A \geq 0, \: \exists \tau \geq 0 $ such that:
  \begin{equation}
  \| X^t - \overline{X} \| \leq A \tau^t
  \end{equation}

    (b) If $ \theta \in \left]\dfrac{1}{2}, 1\right[ $, then $ \exists A \geq 0 $ such that:
  \begin{equation}
  \| X^t - \overline{X} \| \leq A t^{-(1-\theta)/(2\theta -1)}
  \end{equation}

\end{lemma}

\begin{proof}
The proof of the previous Lemma follows from the recent work of Attouch and Bolte. There is no explicit statement of the asymptotic convergence rates in~\cite{Bolte2013}, however, one can strictly follow~\cite{AttouchDemo}.
\end{proof}

   \subsection{Properties}
   \paragraph{Kurdyka-\Lpb{}ojasiewicz}
   \begin{proposition}
\label{KLProp}
The function G defined in~\ref{G} satisfies the Kurdyka-\Lpb{}ojasiewicz Property at all its critical points with a function $\phi(s) = s^{1-\theta}$ where $\theta \in [\dfrac{1}{2},1[$.\\
Let us denote by $U$ and $\eta$ the associated objects in definition~(\ref{KLDefinition}).

\end{proposition}

\begin{proof}
\Lpb{}ojasiewicz (\cite{Lojasiewicz65,Lojasiewicz84}), showed that any real analytic function has the Kurdyka-\Lpb{}ojasiewicz property with  $\phi(s) = s^{1-\theta}$ for some $\theta \in [\dfrac{1}{2},1[$. \\
Our function G is obviously analytic and real. Which terminates the proof of Proposition~(\ref{KLProp}).
\end{proof}

   \begin{lemma}
   The sequence $\{Q^t \}$ belongs to a compact set $\Sigma \subset ]0,1[^N$. Let us define :\\
   $\Sigma :=\prod \limits_{i} [q^{min}_i , q^{max}_i]$
  \end{lemma}

\begin{proof}
We know that $\Psi$ in bounded. Let us define :
     \begin{equation*}
\forall{x \in \{0,1\}^N} \: \Psi_{min} \leq \Psi(x) \leq \Psi_{max}
  \end{equation*}
      \begin{equation}
      \label{qmin}
\left\{1,\ldots,N\right\} \:q^{min}_i = \dfrac{1}{1 + \exp\left(\Psi_{max}- \Psi_{min} + \log\left(\dfrac{1-p^0_i}{p^0_i}\right)\right)}
  \end{equation}
      \begin{equation}
      \label{qmax}
\left\{1,\ldots,N\right\} \:q^{max}_i = \dfrac{1}{1 + \exp\left( \Psi_{min}- \Psi_{max} + \log\left(\dfrac{1-p^0_i}{p^0_i}\right)\right)}
  \end{equation}

 Then, if we assume that $q^t \in [q^{min}_i , q^{max}_i]$, using~\eqref{minformula},~\eqref{qmin},~\eqref{qmax}, we can write the following :
 \begin{align*}
\log\left(\dfrac{1-q^{t+1}}{q^{t+1}}\right)& =  \dfrac{1}{1+\lambda}\left[E_{Q^t_{\setminus i}}(\Psi(X)| X_i = 1) - E_{Q^t_{\setminus i}}(\Psi(X)| X_i = 0) + \log\left(\dfrac{1-p^0_i}{p^0_i}\right)+ \lambda \log\left(\dfrac{1-q^t}{q^t}\right)\right]\\
 & \leq \dfrac{1}{1+\lambda}\left[\Psi_{max}- \Psi_{min} + \log\left(\dfrac{1-p^0_i}{p^0_i}\right)+ \lambda \log\left(\dfrac{1-q^{min}_i }{q^{min}_i }\right)\right] \\
 & \leq \dfrac{1}{1+\lambda}\left[\log\left(\dfrac{1-q^{min}_i }{q^{min}_i }\right)+ \lambda \log\left(\dfrac{1-q^{min}_i }{q^{min}_i }\right)\right] \\
&  \leq \log\left(\dfrac{1-q^{min}_i }{q^{min}_i }\right)\\
 \end{align*}

  By monotonicity and conversely for the upper bound, we conclude, that $q^{t+1} \in [q^{min}_i , q^{max}_i]$.
 Therefore, by induction, as long as $Q^0 \in \Sigma$, ($Q^0 = P^0$ for instance),  $Q^t  \in \Sigma$ $ \forall t$

\end{proof}

\paragraph{Sufficient decrease}

   \begin{lemma}
 \label{lemmastrcvx}
The penalization $l$ (Equation~(\ref{f_definition})) is 1-strongly convex on ]0,1[. Therefore :

      \begin{equation}
      \label{strcvx}
\textit{For all } x \textit{ and } {x_0 \textit{ in }  ]0,1[},\:  \dfrac{1}{2} \|x - x_0\|^2 \leq l(x,x_0)
  \end{equation}
  \end{lemma}

  \begin{proof}
 
  By a simple differentiation of l, we get :
  \begin{equation*}
  \dfrac{ \partial^2 l(x,x_0)}{{\partial x}^2} = \dfrac{1}{x} + \dfrac{1}{(1-x)} \geq 1
  \end{equation*}
  Then, by definition of the strong convexity, combined with $l(x_0,x_0) =0$, and $l^{\prime}(x_0,x_0) =0$, we get the second part of the Lemma.
\end{proof}

\begin{proposition}
\label{sufficient_decrease}
Our alternate minimization algorithm has the following sufficient decrease property~.\\
For all indices $t \geq 1$,
\begin{equation*}
G(Q^{t+1}) + \dfrac{\lambda}{2} \|Q^{t+1} - Q^{t}\|^2 \leq G(Q^{t})
\end{equation*}
\end{proposition}

\begin{proof}

An elementary induction gives, for each step :\\
\begin{equation}
G(\{q^{t+1}_1, \dots,q^{t+1}_{i-1},q^{t+1}_{i},q^{t}_{i+1},q^{t}_N\})  + \lambda l(q^{t+1}_{i},q^{t}_i) \leq G(\{q^{t+1}_1, \dots,q^{t+1}_{i-1},q^{t}_{i},q^{t}_{i+1},q^{t}_N\})
\end{equation}
Therefore, using the same equations for $ i = \{1, \dots,N\}$,  it easily follows : \\
\begin{equation}
\label{decreaseEq}
G(Q^{t+1}) + L(Q^{t+1}, Q^{t}) \leq G(Q^{t})
\end{equation}

And by strong convexity property of Lemma~\ref{lemmastrcvx}, we get :
\begin{equation*}
G(Q^{t+1}) + \dfrac{\lambda}{2} \|Q^{t+1} - Q^{t}\|^2 \leq G(Q^{t})
\end{equation*}

\end{proof}

\paragraph{Gradient bound}
   \begin{lemma}

   \label{omegalips}
$\Omega$, defined in~\ref{OmegaDef} is $K_{\Omega}-Lipschitz$ with $K_{\Omega} = \Psi_{max} \sqrt{N} $. \\
\end{lemma}

 \begin{proof}
 For any $i$ in $1, \dots,N$ :
 \begin{equation*}
\left| \dfrac{\partial \Omega}{\partial q_i}(Q^t)\right| = |E_{Q^t/i}(\Psi(X)| X_i = 1) - E_{Q^t/i}(\Psi(X)| X_i = 0)| \leq \Psi_{max}
 \end{equation*}
 Therefore, using the classical inequality between $L_2$ and $L_{\infty}$ norms :
  \begin{equation*}
\| \nabla \Omega(Q^t)\|  \leq \Psi_{max} \sqrt{N}
 \end{equation*}
 \end{proof}

  \begin{lemma}

 There exists a positive constant $K_l$ such that for any $Q$ and $\tilde{Q}$ in $\Sigma$:
  \begin{equation}
\label{lips}
\forall i \in \{1,...,N\}, \: | {l}^{\prime}(q_i,\tilde{q_i}) | < 2K_l |q_i-\tilde{q_i}|
  \end{equation}
  \end{lemma}

    \begin{proof}
    For any $i$ in $\{0,...,N\}$, the function  $x \rightarrow log(x)$ is Lispschitz continuous on $[q^{min}_i , q^{max}_i]$ with Lipschitz constant $\dfrac{1}{q^{min}_i} $. And the function  $x \rightarrow log(1-x)$ is Lispschitz continuous on $[q^{min}_i , q^{max}_i]$ with Lipschitz constant $\dfrac{1}{1-q^{max}_i} $.\\
    Therefore,  according to Eq.\ref{lprime} 
  \begin{equation*}
\forall{x \in [q^{min}_i , q^{max}_i]},\: \forall{x_0 \in [q^{min}_i , q^{max}_i]} ,\:  | {l}^{\prime}(x,x_0) | < \left(\dfrac{1}{q^{min}_i} +\dfrac{1}{1-q^{max}_i} \right) |q -q_0|
  \end{equation*}
\end{proof}

Therefore, if we simply set $K_l$ such that : $ K_l=\max_{i \in \{1,...,N\}} \left(\dfrac{1}{q^{min}_i} +\dfrac{1}{1-q^{max}_i} \right) $, Eq.\ref{lips} comes directly.\\

  \begin{lemma}
  \label{lemmaGradBound}
  For any index $u \geq 1$, the following bound on the gradient of G holds :
  \begin{equation}
\| \nabla G(Q^u) \| \leq (2K_l + \sqrt{N-1} K_{\Omega}) \|Q^{u} -Q^{u-1} \|
\end{equation}
  \end{lemma}

  \begin{proof}
  Let us choose $u \geq 1$. For any $i$, from the first order minimization condition in Eq.~\ref{procedure}, we know that :\\
  \begin{equation}
0 = \lambda l ^{\prime}(q^{u}_i,q^{u-1}_i) + \dfrac{\partial{G}}{\partial{q_i}}(\{q^{u}_1, \dots,q^{u}_{i-1},q^{u}_i,q^{u-1}_{i+1}, \dots,q^{u-1}_{N}\})
  \end{equation}
 Which we can rewrite as, using the decomposition on $G$ :
\begin{align*}
\label{gradBound1}
0 & = \lambda l ^{\prime}(q^{u}_i,q^{u-1}_i) + \dfrac{\partial{G}}{\partial{q_i}}(\{q^{u}_1, \dots,q^{u}_{i-1},q^{u}_i,q^{u-1}_{i+1}, \dots,q^{u-1}_{N}\}) \\
 & =\lambda l ^{\prime}(q^{u}_i,q^{u-1}_i) + \dfrac{\partial{G}}{\partial{q_i}}(\{q^{u}_1, \dots,q^{u}_{i-1},q^{u}_i,q^{u}_{i+1}, \dots,q^{u}_{N}\}) \\
 & -\dfrac{\partial{\Omega}}{\partial{q_i}}(\{q^{u}_1, \dots,q^{u}_{i-1},q^{u}_i,q^{u}_{i+1}, \dots,q^{u}_{N}   \})-\dfrac{\partial{f_i}}{\partial{q_i}}(q^{u}_i)\\
 & +\dfrac{\partial{\Omega}}{\partial{q_i}}(\{q^{u}_1, \dots,q^{u}_{i-1},q^{u}_i,q^{u-1}_{i+1}, \dots,q^{u-1}_{N}   \}) +\dfrac{\partial{f_i}}{\partial{q_i}}(q^{u}_i)
\end{align*}

Using equation~\eqref{lips}, and the Lipschitz constant $K_{\Omega}$ of $\Omega$ (see lemma \ref{omegalips}), we get:
\begin{equation*}
\dfrac{\partial{G}}{\partial{q_i}}(Q^{u+1}) \leq 2K\|q^{u}_i-q^{u-1}_i\| + K_{\Omega} \|Q^{u} - \{q^{u}_1, \dots,q^{u}_{i-1},q^{u}_i,q^{u-1}_{i+1}, \dots,q^{u-1}_{N} \} \|
\end{equation*}

\begin{equation}
\label{gradBound}
\dfrac{\partial{G}}{\partial{q_i}}(Q^{t}) \leq 2K\|q^{u}_i-q^{u-1}_i\| + K_{\Omega} \|Q^{u} -Q^{u-1} \|
\end{equation}
Combining equation $\ref{gradBound}$  for $ i = \{1, \dots,N-1\}$ we get :

\begin{equation}
\label{nablaBound}
\| \nabla G(Q^u) \| \leq (2K + \sqrt{N-1} K_{\Omega}) \|Q^{u} -Q^{u-1} \|
\end{equation}

  \end{proof}

     \subsection{Convergence}

We showed in the previous section (Lemma \ref{lemmaGradBound}, Proposition \ref{sufficient_decrease} and Proposition \ref{KLProp}) that the sequence generated by our new minimization procedure has the three sufficient properties for convergence, as shown in Lemma~\ref{Central}. Therefore, according to Lemma~\ref{Central} (or~\cite{AttouchSequence}), the main Theorem of this paper can be stated as follows.

  \begin{theorem}[Convergence]
  The sequence $\{ {Q^t} \}$ generated by the proximal alternate minimization procedure described in algorithm \ref{procedure}, converges to a critical point of $F$, denoted $\overline{Q}$.
   \end{theorem}
   
  \begin{corollary}
   The following asymptotic convergence rates hold : \\
  We recall that $\theta$ is the exponent of the $\phi$ function in the Kurdyka-\Lpb{}ojasiewicz inequality such that  $\phi(s) = cs^{1-\theta}$.\\
    (i) If $ \theta \in ]0, \dfrac{1}{2}] $, then $ \exists C \geq 0, \: \exists \tau \geq 0  $ such that:
  \begin{equation}
  \label{exponential_convegence}
  \| Q^t - \overline{Q} \| \leq C \tau^t
  \end{equation}
  \\
    (ii) If $ \theta \in ]\dfrac{1}{2}, 1[ $, then $ \exists C \geq 0 $ such that:
  \begin{equation}
  \| Q^t - \overline{Q} \| \leq C t^{-(1-\theta)/(2\theta -1)}
  \end{equation}

  If we make the standard SSOC assumption on $G$ (the hessian is positive definite at all the local minimas), then we can show that the convergence rate toward the local minima is linear, as in Equation~\ref{exponential_convegence}, with $\theta = 1/2$.\\
  \end{corollary}

  \begin{proof}[Proof of the corollary]
  The first part of the corollary is also a direct consequence of Lemma~\ref{Central}.\\
  Let us now assume that the Hessian matrix is positive definite at all the local minima (SSOC assumption). We denote by $\mu_1$ and $\mu_2$ the highest and lowest eigenvalues of the Hessian in a neighborhood of a local minimum $\overline{Q}$. $\mu_1$ and $\mu_2$  are both positive by SSOC and continuity of the Hessian. Let us then write the Taylor formula for $G$ and $\nabla G$ at the neighborhood of $\overline{Q}$. It follows the existence of a neighborhood $U$ of $\overline{Q}$, so that, for all $Q \in U$:\\
  \begin{equation}
  |G(Q)-G(\overline{Q})| \leq \mu_1 \|Q-\overline{Q}\|^2
  \end{equation}
  and
  \begin{equation}
  \|\nabla G(Q)\| \geq \mu_2 \|Q-\overline{Q}\|
  \end{equation}

  It shows that G follows a Kurdyka-Losajewicz inequality at all its minimal points, with $\phi(s) = c\sqrt{s}$. Therefore, if ${Q} $ converges toward a minimal point, which has the SSOC, the convergence rate is linear with $\theta = 1/2$.
  \end{proof}

\section{Conclusion}
Although  the convergence  of  fixed  point iterations  schemes  for mean  field
minimization is often taken for granted,  no formal proof exists. In this paper,
we have  proposed a  slightly modified  scheme that  is provably  convergent. This  addresses  a major  conceptual  weakness  of  one  of the  most  important algorithms used by the Machine Learning community. \\
Interestingly, our regularisation can be chosen as small as needed through the parameter $\lambda$. Therefore, our algorithm can be arbitrarily similar to the classical minimisation, while guaranteeing convergence. \\
In future work, we will explore the practical applications for our scheme. We will look for examples where it accelerates convergence. It may prevent infinite, but also temporary oscillations between equivalent solutions of a learning optimisation problem.

\bibliography{Biblio1}{}

\end{document}